\documentclass[11pt]{article}

\usepackage{fullpage}

\usepackage[T1]{fontenc}
\usepackage[utf8]{inputenc}
\usepackage{lmodern}
\usepackage{microtype}

\usepackage{amsmath,mathtools}
\usepackage{amsfonts,amssymb}
\usepackage{amsthm,thmtools}
\usepackage{bm}
\usepackage{bbm}
\usepackage{accents}

\usepackage{graphicx}
\usepackage{xcolor}
\usepackage{booktabs}
\usepackage{array}
\usepackage{tabularx}
\usepackage{threeparttable}
\usepackage{caption}
\usepackage{subcaption}
\usepackage{pifont}

\usepackage{enumitem}

\usepackage{algorithm}
\usepackage{algpseudocode}

\usepackage[round]{natbib}
\usepackage[colorlinks=true,citecolor=blue,linkcolor=red,urlcolor=magenta]{hyperref}
\usepackage{cleveref}

\crefformat{equation}{(#2#1#3)}
\Crefformat{equation}{(#2#1#3)}
\crefrangeformat{equation}{(#3#1#4)--(#5#2#6)}
\Crefrangeformat{equation}{(#3#1#4)--(#5#2#6)}
\crefmultiformat{equation}{(#2#1#3)}{ and (#2#1#3)}{, (#2#1#3)}{, and (#2#1#3)}
\Crefmultiformat{equation}{(#2#1#3)}{ and (#2#1#3)}{, (#2#1#3)}{, and (#2#1#3)}

\newtheorem{theorem}{Theorem}
\newtheorem{lemma}[theorem]{Lemma}
\newtheorem{proposition}[theorem]{Proposition}

\newtheorem{assumption}[theorem]{Assumption}
\newtheorem{remark}[theorem]{Remark}

\crefname{assumption}{Assumption}{Assumptions}
\Crefname{assumption}{Assumption}{Assumptions}

\newcommand{\para}[1]{\paragraph{#1.}}

\newcommand{\R}{\mathbb{R}}

\newcommand{\I}{\mathbf{I}}

\DeclarePairedDelimiterX{\cbrp}[1]{\lbrace}{\rbrace}{#1}
\DeclarePairedDelimiterX{\brp}[1]{(}{)}{#1}
\DeclarePairedDelimiterX{\sqbrp}[1]{[}{]}{#1}
\DeclarePairedDelimiterX{\normp}[1]{\lVert}{\rVert}{#1}
\DeclarePairedDelimiterX{\absp}[1]{\lvert}{\rvert}{#1}
\DeclarePairedDelimiterX{\ipp}[1]{\langle}{\rangle}{#1}

\newcommand{\cbr}[1]{\ensuremath{\mathchoice{\cbrp*{#1}}{\cbrp{#1}}{\cbrp{#1}}{\cbrp{#1}}}}
\newcommand{\br}[1]{\ensuremath{\mathchoice{\brp*{#1}}{\brp{#1}}{\brp{#1}}{\brp{#1}}}}
\newcommand{\sqbr}[1]{\ensuremath{\mathchoice{\sqbrp*{#1}}{\sqbrp{#1}}{\sqbrp{#1}}{\sqbrp{#1}}}}
\newcommand{\norm}[1]{\ensuremath{\mathchoice{\normp*{#1}}{\normp{#1}}{\normp{#1}}{\normp{#1}}}}
\newcommand{\abs}[1]{\ensuremath{\mathchoice{\absp*{#1}}{\absp{#1}}{\absp{#1}}{\absp{#1}}}}

\DeclareMathOperator*{\argmax}{arg\,max}
\DeclareMathOperator*{\argmin}{arg\,min}

\DeclareMathOperator{\diag}{diag}

\renewcommand{\P}{\mathbb{P}}
\newcommand{\E}{\mathbb{E}}

\newcommand{\ind}[1]{\ensuremath{\mathbf{1}\,\cbr{#1}}}
\DeclareMathOperator{\Bern}{Bern}

\DeclareMathOperator{\Unif}{Unif}

\newcommand{\KL}{D_{\mathrm{KL}}}

\renewcommand{\v}[1]{\boldsymbol{#1}}
\newcommand{\m}[1]{\mathbf{#1}}
\newcommand{\tp}{^{\mathsf{T}}}

\renewcommand{\cal}[1]{\mathcal{#1}}

\newcommand{\what}[1]{\widehat{#1}}
\newcommand{\wtilde}[1]{\widetilde{#1}}

\DeclareRobustCommand{\term}[1]{%
  \ifmmode
    A_{\text{#1}}%
  \else
    term~\ensuremath{A_{\text{#1}}}%
  \fi
}

\newcommand{\eps}{\epsilon}

\DeclarePairedDelimiterX{\parens}[1]{(}{)}{#1}

\algnewcommand\algorithmicinit{\textbf{initialize}}
\algnewcommand\Init{\item[\algorithmicinit]}
\algnewcommand\algorithmicinput{\textbf{input}}
\algnewcommand\Input{\item[\algorithmicinput]}

\newcommand{\note}[1]{{\color{blue}#1}}

\title{
Algorithm for Contextual Queueing Bandits with Rate-Optimal Queue Length Regret
}

\author{
Seoungbin Bae$^{1}$ \quad Dabeen Lee$^{2}$\\[0.3em]
$^{1}$Department of Industrial \& Systems Engineering, KAIST\\
$^{2}$Department of Mathematical Sciences, Seoul National University\\
\texttt{sbbae31@kaist.ac.kr, dabeenl@snu.ac.kr}
}

\date{}

\begin{document}
\maketitle

\begin{abstract}
Contextual queueing bandits provide a framework for learning to schedule heterogeneous jobs under unknown context-dependent service rates.
Under stochastic contexts, existing algorithms achieve $\wtilde{\cal O}(T^{-1/4})$ queue length regret, defined as the expected difference between the learner's and oracle's queue lengths at horizon $T$.
In this paper, we improve this rate to $\wtilde{\cal O}(T^{-1/2})$. The key observation is that random exploration is needed only up to a carefully chosen cutoff round, rather than throughout the entire horizon.
We propose CQB-$\eta$-2, a three-phase algorithm: (i) pure random exploration to construct an initial estimator, (ii) $\eta$-random exploration combined with a UCB rule to continue learning while maintaining negative drift, and (iii) pure UCB after the exploration cutoff.
Our proof decomposes the queue length regret at the cutoff round. Before the cutoff, negative drift suppresses queue length differences caused by suboptimal choices. After the cutoff, the first two phases provide sufficient random exploration samples, ensuring that UCB decisions incur small departure-rate gaps. Combining these two bounds yields queue length regret of order $\wtilde{\cal O}(T^{-1/2})$.
We further prove a minimax lower bound of order $\Omega(T^{-1/2})$. The proof constructs two hard instances that are statistically indistinguishable up to the final service decision, and uses a queue-specific coupling argument to convert the resulting testing error into queue length regret. Together, our upper and lower bounds characterize the minimax dependence on the horizon $T$ up to logarithmic factors.
\end{abstract}

\section{Introduction}
\label{sec:introduction}

Queueing systems play an important role in modern service platforms, including cloud computing \citep{vilaplana2014queuing}, online service systems \citep{andrews2004scheduling}, call centers \citep{koole2002queueing}, and multi-LLM services \citep{mitzenmacher2025queueing}.
In these systems, jobs arrive over time, service capacity is limited, and scheduling decisions must be made online.
A central challenge is that jobs are heterogeneous and carry contexts, while the corresponding context-dependent service rates are often unknown in advance.
This motivates learning-while-scheduling policies that infer unknown service rates from observed departures while keeping the queue stable \citep{krishnasamy2016regret,kim2024queueing}.

Queueing bandits formalize learning-while-scheduling when service rates are unknown \citep{krishnasamy2016regret,krishnasamy2021learning}.
In a discrete-time queueing bandit problem, the learner selects which server or job-server pair to use, observes binary departure feedback, and uses this feedback to improve future scheduling decisions.
The objective is not only to identify choices with large departure probabilities, but also to control the queue length, because each suboptimal decision can affect future queue states.
However, classical queueing bandits typically do not account for job contexts: each action has a fixed unknown departure probability, independent of the features of the jobs currently in the queue.

Contextual queueing bandits address this limitation by combining queueing bandits with contextual bandit models.
In this framework, the departure probability of a job-server pair is modeled as a logistic function of its feature vector and an unknown parameter \citep{bae2026queue,bae2026learning}, as in logistic and generalized linear bandits \citep{filippi2010parametric,li2017provably,faury2020improved}.

In this paper, we study contextual queueing bandits under stochastic contexts in a discrete-time system with a single queue and $K$ servers.
In each round, the learner observes the contexts of jobs in the queue, selects a job-server pair to process, observes whether the selected job departs, and a new job may arrive.
The performance measure is queue length regret, defined as the expected difference between the learner's and oracle's queue lengths at horizon $T$, where the oracle policy knows the true departure probabilities.
Unlike standard bandit regret, which accumulates reward losses over time, queue length regret compares the queue states at horizon $T$.
The closest prior works on contextual queueing bandits \citep{bae2026queue,bae2026learning} establish $\wtilde{\cal O}(T^{-1/4})$ queue length regret under stochastic contexts.
This leaves open whether the $T^{-1/4}$ rate is intrinsic to contextual queueing bandits or a consequence of the existing algorithmic design and analysis.

We answer this question by showing that the queue length regret can be improved to $\wtilde{\cal O}(T^{-1/2})$.
Our algorithm, CQB-$\eta$-2, is based on the principle that random exploration should be used to learn the unknown service model, but need not be maintained until the end of the horizon.
This differs from previous algorithms, which maintain random exploration throughout the entire horizon.
We also prove a minimax lower bound of order $\Omega(T^{-1/2})$, showing that the dependence on $T$ is tight up to logarithmic factors.

Our contributions are summarized as follows:
\begin{itemize}
    \item We propose CQB-$\eta$-2, a three-phase contextual queueing bandit algorithm motivated by the idea of cutting off random exploration after sufficient samples are collected.
    The algorithm first performs pure random exploration to construct an initial estimator, then uses a mixture of $\eta$-random exploration and a UCB rule to continue learning while maintaining negative drift, and finally cuts off random exploration and uses only the UCB rule.

    \item We prove that CQB-$\eta$-2 achieves queue length regret of order $\wtilde{\cal O}(T^{-1/2})$ under stochastic contexts, improving the previous $\wtilde{\cal O}(T^{-1/4})$ rate.
    This shows that random exploration until the end of the horizon is not necessary for achieving the $T^{-1/2}$ order.

    \item The improvement in queue length regret is due to the clever design of phase transitions. In particular, we stop pure exploration at some cutoff point. First, although the learner may still make suboptimal decisions before the cutoff point due to random exploration, we observe that a negative drift makes the resulting effect on queue length small towards the end of horizon $T$.
    After the cutoff point, we would collect $\Theta(T)$ random exploration samples, in which case pure UCB decisions have departure-rate gaps of order $T^{-1/2}$.
    This separation explains why stopping random exploration can improve the queue length regret rate.

    \item We establish a minimax lower bound of order $\Omega(T^{-1/2})$.
    The proof constructs two hard  
    instances that are statistically indistinguishable up to the final service decision and uses a queue-specific coupling argument to convert the resulting testing error into queue length regret.
    This shows that the dependence on $T$ is tight up to logarithmic factors.

    \item We provide simulations demonstrating that CQB-$\eta$-2 achieves lower empirical queue lengths than baseline algorithms.
\end{itemize}

\section{Preliminaries}
\label{sec:prelim}

This section reviews the contextual queueing bandit model, the definition of queue length regret, and the policy-switching queue/coupling framework, following \citet{bae2026queue,bae2026learning}.

\subsection{Model}
We consider a discrete-time contextual queueing system with a single queue and $K$ servers, where in each round the agent selects a job-server pair to process and a new job may arrive. At the beginning of round $t$, the queue state $\cal X_t$ is the collection of contexts of all remaining jobs. The queue length is $Q(t)=\abs{\cal X_t}$. If $\cal X_t\neq\emptyset$, the agent chooses a job-server pair $(x_t,a_t)\in\cal X_t\times[K]$. If $\cal X_t=\emptyset$, the agent selects a fixed dummy pair $(x_0,a_0)$, and the resulting observation is not used for learning.
Let $A(t)\in\cbr{0,1}$ be the arrival random variable in round $t$, with $\P\br{A(t)=1}=\lambda$, and let $x^{(t)}$ be the newly arriving job context when $A(t)=1$. Let $D(t)\in\cbr{0,1}$ be the departure random variable in round $t$. We will specify its conditional distribution below. Then the queue state and queue length evolve according to
\begin{align*}
    \cal X_{t+1}=\cal X_t\setminus\cbr{x_t:D(t)=1}\cup\{x^{(t)}:A(t)=1\},
    \quad
    Q(t+1)=\sqbr{Q(t)+A(t)-D(t)}^+.
\end{align*}
Here, $\sqbr{z}^+=\max\cbr{0,z}$. We use $\v A(t)=(A(t),\wtilde x^{(t)})$, where $\wtilde x^{(t)}=x^{(t)}$ if $A(t)=1$ and $\wtilde x^{(t)}=\wtilde x$ otherwise for a fixed no-arrival symbol $\wtilde x$. We also write $\v D(t)=(D(t),(x_t,a_t))$. Let $E(t)\in\cbr{0,1}$ be the exploration random variable used in round $t+1$. The filtration is $\cal F_t=\sigma\br{\cal X_1,\v A(1),\v D(1),E(1),\ldots,\v A(t-1),\v D(t-1)}$. We assume that arrivals are exogenous with conditional arrival probability $\P\br{A(t)=1\mid\cal F_t}=\lambda$. For a filtration $\cal G$ and a random variable $Z$, we write $\E\sqbr{\cdot\mid\cal G,Z}$ for conditioning on $\cal G\lor\sigma(Z)$. For each job-server pair $(x,a)$, let $\v\phi(x,a)\in\R^d$ be its feature vector, and let $\v\theta^*\in\R^d$ be the unknown parameter. Then, the departure probability follows
\begin{align*}
    D(t)\mid \cal F_t,x_t,a_t \sim \Bern\br{\mu(\langle\v\phi(x_t,a_t),\v\theta^*\rangle)}, \quad \mu(z)= (1+\exp(-z))^{-1}.
\end{align*}

\subsection{Oracle policy and queue length regret}
Let $\pi$ denote our policy, and let $\pi^*$ be the oracle policy that knows $\v\theta^*$. Given a nonempty queue state $\cal Y$, $\pi^*$ chooses a job-server pair that maximizes $\mu(\langle\v\phi(x,a),\v\theta^*\rangle)$ over $(x,a)\in\cal Y\times[K]$. The queue length regret at horizon $T$ is $R_T=\E\sqbr{Q(T)-Q^*(T)}$, where $Q^*(T)$ is the queue length under $\pi^*$.

\subsection{Policy-switching queues and coupling}

Let $\pi$ be a given policy. For $t\in[0,T-1]$, let $\pi_t$ be the policy that follows $\pi$ for rounds $1,\ldots,t$ and then follows $\pi^*$ for rounds $t+1,\ldots,T-1$. We construct a coupling of the policy-switching queueing processes as follows. All processes use the same arrival random variables $\v A(i)$ and the same exploration random variables $E(i)$; these exploration random variables affect only rounds in which the process follows $\pi$. Let $Q_t(i)$ be the queue length at the beginning of round $i$ in the coupled process governed by $\pi_t$, and let $(x_{t,i},a_{t,i})$ be the job-server pair selected in round $i$ by this process. In each round $i$, draw a shared random variable $U_i\sim\Unif(0,1)$. Let $D_t(i)=\ind{U_i\leq\mu(\langle\v\phi(x_{t,i},a_{t,i}),\v\theta^*\rangle)}$ and write $\v D_t(i)=\br{D_t(i),(x_{t,i},a_{t,i})}$. For notational convenience, we use $\v D_t(i)=0$ and $\v D_t(i)=1$ as shorthand for the events $D_t(i)=0$ and $D_t(i)=1$, respectively, together with the selected job-server pair $(x_{t,i},a_{t,i})$.
This construction gives $\E[Q(T)]=\E[Q_{T-1}(T)]$ and $\E[Q^*(T)]=\E[Q_0(T)]$, while consecutive processes governed by $\pi_t$ and $\pi_{t-1}$ have the same queue state at the beginning of round $t$. Define $\psi_t(T)= Q_t(T)-Q_{t-1}(T)$ for $t\in[T-1]$. Then, we can decompose the queue length regret as
\begin{align*}
    R_T= \E[Q_{T-1}(T)-Q_0(T)]=\sum_{t=1}^{T-1}\E\sqbr{\psi_t(T)}.
\end{align*}
\subsection{Assumptions}
Throughout the paper, we impose the following standard conditions.
\begin{assumption} \label{ass:basic}
For all $x\in\cal X$ and $a\in[K]$, $\norm{\v\phi(x,a)}_2\leq1$. Also, for a known $S>0$, $\v\theta^*\in\Theta=\cbr{\v\theta\in\R^d:\norm{\v\theta}_2\leq S}$.
\end{assumption}

\begin{assumption} \label{ass:constants}
There exist $\kappa,R>0$ such that $1/\kappa\leq\dot\mu(\langle\v\phi(x,a),\v\theta\rangle)\leq R$ for all $x\in\cal X$, $a\in[K]$, and $\v\theta\in\Theta$.
\end{assumption}

\begin{assumption} \label{ass:iid}
Newly arriving contexts are i.i.d. from an unknown distribution $\cal D$. Moreover, for some $\sigma_0^2>0$, $\lambda_{\min}(\E_{x\sim\cal D}\sqbr{\frac{1}{K}\sum_{a\in[K]}\v\phi(x,a)\v\phi(x,a)\tp})\geq \sigma_0^2$.
\end{assumption}

\begin{assumption} \label{ass:slackness}
There exists a traffic slackness parameter $\eps>0$ such that for every $x\in\cal X$, there exists $a^*(x)\in[K]$ satisfying $\mu(\langle\v\phi(x,a^*(x)),\v\theta^*\rangle)-\lambda\geq\eps$.
\end{assumption}

\Cref{ass:basic,ass:constants} are standard in logistic bandits \citep{filippi2010parametric}.
\Cref{ass:iid} guarantees sufficient feature diversity, which is also standard in finite-armed logistic bandits \citep{li2017provably}.
\Cref{ass:slackness} imposes a traffic slackness condition that guarantees negative drift under the oracle policy, and similar conditions can be found in \citet{krishnasamy2016regret,kim2024queueing,bae2026queue}.

\section{Improved queue length regret via random exploration cutoff}
\label{sec:algorithm}

\subsection{Motivation and comparison with previous works}
The closest prior works on contextual queueing bandits, due to \citet{bae2026queue,bae2026learning}, obtain $\wtilde{\cal O}(T^{-1/4})$ queue length regret.
There are two reasons to expect that this rate can be improved.
First, since queue length regret is evaluated at a terminal horizon rather than accumulated over all rounds, the $\sqrt{T}$ cumulative-regret scale in logistic bandits suggests the possibility of a $T^{-1/2}$ rate for this terminal-horizon performance measure.
Second, the algorithms of \citet{bae2026queue,bae2026learning} both have to maintain random exploration until the end of the horizon.
This requirement comes from the proof technique.
To elaborate, by the decomposition result in \Cref{lem:decomp}, previous proofs control the departure-rate gap in the first square-root factor by a monotone upper bound, and maintaining random exploration until the end of the horizon preserves this monotone control.
From an algorithmic perspective, however, after sufficiently accurate estimation is achieved through early random exploration, it is natural to terminate random exploration and use only the UCB rule.
Motivated by these two observations, we propose a three-phase algorithm that cuts off random exploration in the final phase.

\subsection{Proposed algorithm}

\begin{algorithm}[t]
\caption{CQB-$\eta$-2}
\label{alg:cqb-eta-2}
\begin{algorithmic}[1]
    \Input $d$, $T$, $K$, $S$, $\kappa$, $\lambda_0$, $\lambda$, $\delta$, $\eps$, $\sigma_0$
    \Init $\tau_1$, $\tau_2$ as in \Cref{eq:tau}, $\m V_0\gets\lambda_0\I$, $\what{\v\theta}_0\gets \v 0$
    \For{$t=1,\dots,T$}
        \If{$t\leq \tau_1$} \note{\Comment{Phase~1: pure random exploration}}
            \State $\eta\gets 1$
        \ElsIf{$\tau_1<t\leq\tau_2$} \note{\Comment{Phase~2: random exploration with probability $\eta$ and UCB}}
            \State $\eta\gets \eps/4$
        \Else   \note{\Comment{Phase~3: pure UCB}}
            \State $\eta\gets 0$
        \EndIf
        \State Sample $E(t-1)\sim\Bern(\eta)$
        \If{$A(t-1)=1$ \textbf{and} $E(t-1)=1$}
            \State $x_t\leftarrow x^{(t-1)}$, \quad $a_t\sim\Unif([K])$
        \Else
            \State $(x_t,a_t)\leftarrow\argmax_{x\in\cal X_t,\;a\in[K]}\mathrm{UCB}_{t}(x,a)$~~~~($\mathrm{UCB}_{t}(x,a)$ as in \Cref{eq:ucb})
        \EndIf
        \State Choose $(x_t,a_t)$ and observe $D(t)$
        \State $\m V_t\leftarrow\m V_{t-1}+\v\phi(x_t,a_t)\v\phi(x_t,a_t)\tp$
        \State Update $\what{\v\theta}_t$ as in \Cref{eq:mle}, $\beta_t$ as in \Cref{eq:beta}
    \EndFor
\end{algorithmic}
\end{algorithm}

\Cref{alg:cqb-eta-2} presents our CQB-$\eta$-2 algorithm.
We use the term \emph{random exploration} for the rule that, when a new job arrives ($A(t-1)=1$) and the exploration flag is set ($E(t-1)=1$), selects the newly arrived job and chooses a server uniformly at random.
The algorithm has three phases.
In Phase 1, it performs random exploration with probability one on arrivals.
In Phase 2, it combines random exploration with probability $\eta$ and a UCB rule, where we set $\eta=\eps/4$.
In Phase 3, it cuts off random exploration and uses only the UCB rule.

The phase transition points $\tau_1$ and $\tau_2$ are defined in the regret analysis in \Cref{sec:upper_bound}.
At a high level, $\tau_1$ ensures that pure random exploration in Phase 1 provides enough information for the rule that combines random exploration with probability $\eta$ and UCB in Phase 2 to guarantee a negative drift.
The transition point $\tau_2$ ensures that the same Phase 2 rule achieves sufficient estimation accuracy, so that the pure UCB rule is sufficient in Phase 3.

For any job-server pair $(x,a)$, the UCB is calculated as
\begin{align}
    \mathrm{UCB}_{t}(x,a)
    =
    \mu(\langle\v\phi(x,a),\what{\v\theta}_{t-1}\rangle)
    +
    \beta_{t-1}\norm{\v\phi(x,a)}_{\m V_{t-1}^{-1}} .
    \label{eq:ucb}
\end{align}
We use the regularized logistic maximum likelihood estimator
\begin{align}
    \what{\v\theta}_t
    \gets
    \argmin_{\v\theta\in\Theta}
    \sum_{i=1}^{t}
    \sqbr{
        -D(i)\log \mu(\langle\v\phi_i,\v\theta\rangle)
        -(1-D(i))\log\br{1-\mu(\langle\v\phi_i,\v\theta\rangle)}
    }
    +
    \frac{\lambda_0}{2}\norm{\v\theta}_2^2 .
    \label{eq:mle}
\end{align}
and the confidence radius given by
\begin{align}
    \beta_t
    \gets
    \kappa\sqrt{0.5 d \log(1+t/(\kappa \lambda_0 d)) + 0.5\log(1/\delta)} + S\sqrt{\kappa \lambda_0} .
    \label{eq:beta}
\end{align}
We adopt this simple form of the estimator and confidence radius because they are not the main focus of our contribution.

Finally, we provide the main theorem of \Cref{alg:cqb-eta-2}.
\begin{theorem} \label{thm:upper_bound}
    Suppose that \Cref{ass:basic,ass:constants,ass:iid,ass:slackness} hold.
    Set $\delta=T^{-4}$.
    For every $T$ satisfying $\tau_2-\tau_1\geq T/2$, the queue length regret of \Cref{alg:cqb-eta-2} satisfies
    \begin{align*}
        R_T
        =
        \cal O\br{
        \br{
            \frac{d+\log(T)}{\lambda\eps\sigma_0^4}
            +
            \frac{\kappa^2d\log(T)}{\lambda\eps^3\sigma_0^2}
            +
            \frac{\log(T)}{\lambda\eps}
            +
            \frac{\log(T)}{\eps^2}
        }
        \frac{\kappa\sqrt{d\log(T)}}
        {\sqrt{\lambda\eps}\,\sigma_0\sqrt{T}}
        }.
    \end{align*}
\end{theorem}

\subsection[Proof of the upper bound]{Proof of \cref{thm:upper_bound}} \label{sec:upper_bound}

The improved queue length regret comes from how the policy-switching decomposition is used.
The decomposition itself is from \citet{bae2026queue}, but previous analyses use it with a monotone departure-rate gap bound over the entire horizon.
This requires random exploration until round $T$.
We instead split the regret at the exploration cutoff.
For rounds before the cutoff, we control the effect at horizon $T$ of a queue length difference by negative drift.
For rounds after the cutoff, we control the departure-rate gap by the estimation accuracy obtained from the $\Theta(T)$ random exploration samples collected before the cutoff.
We now recall the decomposition and then make this split explicit.
For the queue state $\cal X_t$ under $\pi$, define
\((x_t^*,a_t^*)\in\argmax_{x\in\cal X_t,a\in[K]}\mu(\langle\v\phi(x,a),\v\theta^*\rangle)\).
Let $\v\phi_t^*=\v\phi(x_t^*,a_t^*)$ and $\v\phi_t=\v\phi(x_t,a_t)$.
Also define $\cal F_t^+=\cal F_t\lor\sigma(E(t-1),\v A(t))$ and
\begin{align*}
    \wtilde\psi_t(T)=\E[\psi_t(T)\mid \cal F_t^+,\v D_t(t)=0,\v D_{t-1}(t)=1].
\end{align*}
The following decomposition corresponds to Lemma 4.2 of \citet{bae2026queue}.
It separates the effect of the decision in round $t$ into the departure-rate gap in that round and the effect at horizon $T$ of the queue-length difference created in that round.
\begin{lemma} \label{lem:decomp}
    For each $t\in[T-1]$, we have
    \begin{align*}
        \E\sqbr{\psi_t(T)}
        \leq
        \sqrt{\E\Big[\br{\mu(\langle\v\phi_t^*,\v\theta^*\rangle)-\mu(\langle\v\phi_t,\v\theta^*\rangle)}^2\Big]}\cdot
        \sqrt{\E\Big[\wtilde\psi_t(T)\Big]}.
    \end{align*}
\end{lemma}
In \Cref{lem:decomp}, the first square-root factor is the departure-rate gap, and the second square-root factor is the effect at horizon $T$ of the queue length difference created in round $t$.
The key point is that these two factors do not have to be controlled by the same type of bound in all rounds.
Before the cutoff, we use a naive upper bound on the departure-rate gap, but the second factor is small by negative drift.
After the cutoff, we use a naive upper bound on the second factor, but the departure-rate gap is small because the first two phases have already collected $\Theta(T)$ random exploration samples.
This is the step that leads to the $T^{-1/2}$ rate.

We also introduce a lemma showing that, under our coupling construction, the two consecutive policy-switching processes governed by $\pi_t$ and $\pi_{t-1}$ differ in their policy only at round $t$, and hence their queue states can differ by at most one in all subsequent rounds.
\begin{lemma} \label{lem:psi}
    For each $t\in[T-1]$, we have $\psi_t(T)\in\{-1,0,1\}$.
\end{lemma}

Next, define the phase transition points $\tau_1,\tau_2$ as follows: For sufficiently large constants $c_1,c_2>0$,
\begin{align}
    \tau_1
    =
    \frac{2}{\lambda}
    \sqbr{
    \bigg(\frac{c_1\sqrt{d} + c_2\sqrt{\log(1/\delta)}}{\sigma_0^2}\bigg)^2
    +
    \frac{128\beta_T^2}{\eps^2\sigma_0^2}
    }
    +
    \frac{8\log(1/\delta)}{\lambda},\quad
    \tau_2
    =
    T - \frac{4\tau_1}{\eps} - \frac{128\log(T)}{\eps^2} -1. \label{eq:tau}
\end{align} 
For the upper-bound proof, we set $\delta=T^{-4}$ in $\beta_t$ and $\tau_1$.
We now bound the two terms in \Cref{lem:decomp}.
We first show that the first term, the departure-rate gap term, is upper bounded at order $T^{-1/2}$ after $\tau_2$, that is, after Phase 2 has collected enough samples of order $\Theta(T)$.
\begin{lemma} \label{lem:late_departure_gap}
    If $\tau_2-\tau_1\geq T/2$, then for every $t\in[\tau_2+1,T-1]$, we have
    \begin{align*}
        \sqrt{\E[\br{\mu(\langle\v\phi_t^*,\v\theta^*\rangle)-\mu(\langle\v\phi_t,\v\theta^*\rangle)}^2]}
        \leq
        \sqrt{
            T^{-4}
            +
            \nu_{\mathrm{ph2}}
            +
        \frac{128\beta_T^2}{\lambda\eps\sigma_0^2T}
        }.
    \end{align*}
    where $\nu_{\mathrm{ph2}}
    =
    \exp\br{-\frac{\lambda\eps(\tau_2-\tau_1)}{32}}
    +
    d\exp\br{-\frac{\lambda\eps(\tau_2-\tau_1)\sigma_0^2}{64}}$.
\end{lemma}
Next, we show that the second term in the decomposition in \Cref{lem:decomp}, the effect of a queue length difference created in round $t$, is upper bounded by an exponentially increasing sequence clipped at one by \Cref{lem:psi}.
\begin{lemma} \label{lem:effect_gap}
    For $t> T - 4\tau_1/\eps - 1$, we have $\sqrt{\E\sqbr{\wtilde\psi_t(T)}} \leq 1$. For $t\leq T - 4\tau_1/\eps - 1$, we have
    \begin{align*}
        \sqrt{\E\sqbr{\wtilde\psi_t(T)}} \leq \min \cbr{1,~\sqrt{3T^{-4} + 19 \eps^{-2} \exp\br{-\frac{\eps^2}{32} \br{T-t-1 - 4\tau_1/\eps}}}}.
    \end{align*}
\end{lemma}
The number of rounds with positive drift up to round $t$ affects the threshold value $4\tau_1/\eps$ that separates the two cases in \Cref{lem:effect_gap}.
Notice that this threshold value does not depend on $\tau_2$.
This is because our choice of $\tau_1$ and the exploration rule in Phase 2 allow the algorithm to collect sufficiently many samples of order $\Theta(T)$ through random exploration while maintaining negative drift during Phase 2.
This is the key reason why, unlike previous works, our algorithm can cut off random exploration in the final phase.
Now, we are ready to start the proof of \Cref{thm:upper_bound}.

\begin{proof} [Proof of \Cref{thm:upper_bound}]

Recall the queue length regret decomposition of $R_T=\sum_{t=1}^{T-1} \E[\psi_t(T)]$. We analyze the queue length regret before and after $\tau_2$ by writing $R_T=\term 1+\term 2$, where $\term 1= \sum_{t=1}^{\tau_2}\E\sqbr{\psi_t(T)}$ and $\term 2=\sum_{t=\tau_2+1}^{T-1}\E\sqbr{\psi_t(T)}$.
This split follows the two different uses of \Cref{lem:decomp}.
In $\term 1$, we control the effect of a queue-length difference by negative drift.
In $\term 2$, we control the departure-rate gap by the estimation accuracy obtained before the cutoff.
For $\term 1$, we use the decomposition in \Cref{lem:decomp} and consider the departure-rate gap term and the effect of a queue length difference term for $t\in[1,\tau_2]$.
For the departure-rate gap term, we use a naive upper bound of one as $\sqrt{\E\sqbr{\br{\mu(\langle\v\phi_t^*,\v\theta^*\rangle)-\mu(\langle\v\phi_t,\v\theta^*\rangle)}^2}}
    \leq
    1$.
For the effect of a queue length difference term,
\begin{align*}
    T-t-1-4\eps^{-1}\tau_1
    \geq
    T-\tau_2-1-4\eps^{-1}\tau_1
    =
    128\eps^{-2}\log(T)
\end{align*}
for every $t\leq\tau_2$.
Thus \Cref{lem:effect_gap} gives $\sqrt{\E\sqbr{\wtilde\psi_t(T)}} \leq \sqrt{3+19\eps^{-2}}\,T^{-2}$.
Therefore,
\begin{align}
    \term 1
    \leq
    \sum_{t=1}^{\tau_2}\sqrt{3+19\eps^{-2}}\,T^{-2}
    \leq \sqrt{3+19\eps^{-2}}\,T^{-1}.
    \label{eq:proof_upper_A1}
\end{align}
Next, consider $\term 2$.
We again apply \Cref{lem:decomp} for $t\in[\tau_2+1,T-1]$.
For the departure-rate gap term, we use \Cref{lem:late_departure_gap} to upper bound it by
\begin{align*}
    \sqrt{\E[\br{\mu(\langle\v\phi_t^*,\v\theta^*\rangle)-\mu(\langle\v\phi_t,\v\theta^*\rangle)}^2]}
    \leq
    \sqrt{
        T^{-4}
        +
        \nu_{\mathrm{ph2}}
        +
        \frac{128\beta_T^2}{\lambda\eps\sigma_0^2T}
    }.
\end{align*}
For the effect of a queue length difference term, we use the naive bound from \Cref{lem:effect_gap} to upper bound it by one as $\sqrt{\E\sqbr{\wtilde\psi_t(T)}} \leq 1$.
Therefore,
\begin{align}
    \term 2
    \leq
    \sum_{t=\tau_2+1}^{T-1}
    \br{
        \sqrt{
            T^{-4}
            +
            \nu_{\mathrm{ph2}}
            +
            \frac{128\beta_T^2}{\lambda\eps\sigma_0^2T}
        }
        \cdot 1
    } = \br{\frac{4\tau_1}{\eps}+\frac{128\log(T)}{\eps^2}}
    \sqrt{
        T^{-4}
        +
        \nu_{\mathrm{ph2}}
        +
        \frac{128\beta_T^2}{\lambda\eps\sigma_0^2T}
    }.
    \label{eq:proof_upper_A2_pre}
\end{align}
For $\nu_{\mathrm{ph2}}$, by the assumption of \Cref{thm:upper_bound},
$\tau_2-\tau_1\geq T/2$, so
$\nu_{\mathrm{ph2}}\le \exp(-\lambda\eps T/64)
+d\exp(-\lambda\eps\sigma_0^2T/128)$. Moreover, the definition of
$\tau_2$ and $\tau_2-\tau_1\ge T/2$ imply
$T\ge 2(1+4/\epsilon)\tau_1$. Since $\tau_1\ge 8\log(1/\delta)/\lambda$
and $\delta=T^{-4}$, we have
$\lambda\epsilon T/64\ge\lambda\tau_1/8\ge4\log T$. Also, using
$\sigma_0^2\le1$ and taking $c_1,c_2$ sufficiently large,
$\lambda\epsilon\sigma_0^2T/128\ge\lambda\sigma_0^2\tau_1/16
\ge\log d+4\log T$. Hence $\nu_{\mathrm{ph2}}\le 2T^{-4}$.
Substituting this into (6), using $\sqrt{a+b}\le\sqrt a+\sqrt b$ and
$4\tau_1/\epsilon+128\log(T)/\epsilon^2=T-\tau_2-1\le T$, gives $A_2
\le
\left(
\frac{4\tau_1}{\epsilon}
+
\frac{128\log(T)}{\epsilon^2}
\right)
\frac{\sqrt{128}\beta_T}{\sqrt{\lambda\epsilon}\sigma_0\sqrt T}
+
O(T^{-1})$, which completes the proof.
\end{proof}


\section{Queue length regret lower bound analysis}
\label{sec:lower_bound}

We now prove that the $T^{-1/2}$ dependence in \Cref{thm:upper_bound} is tight in terms of the horizon $T$.
The proof uses a two-instance testing argument.
The two instances have the same arrival process and context distribution, but their optimal server sets are complementary.
They are statistically difficult to distinguish from the observed history before round $T-1$, while the service decision in round $T-1$ can still change the terminal queue length.
We compare an arbitrary policy with a comparison policy that changes only the server chosen in round $T-1$, and then use a testing argument conditioned on the event that a job arrives in round $T-2$ to convert the testing error into queue length regret.

For an instance $\nu$ and a policy $\pi$, let $Q^\pi(T)$ be the queue length at horizon $T$ under $\pi$, and let $Q^*(T)$ be the queue length at horizon $T$ under the oracle policy.
Write $R_T(\pi;\nu)= \E_\nu\sqbr{Q^\pi(T)-Q^*(T)}$.
For fixed $d,K,S,\sigma_0^2,\lambda$, and $\eps$, let $\cal C(d,K,S,\sigma_0^2,\lambda,\eps)$ be the class of contextual queueing bandit instances satisfying \Cref{ass:basic,ass:constants,ass:iid,ass:slackness} with dimension $d$, $K$ servers, parameter radius $S$, feature diversity parameter $\sigma_0^2$, arrival rate $\lambda$, and traffic slackness parameter $\eps$.

\begin{theorem}
\label{thm:lower_bound}
Fix $d\geq2$, $K\geq2$, $\eps\in(0,1)$, and $\lambda\in(0,1-\eps)$.
Let $0<\sigma_0^2\leq 1/d$, and let $B_{\lambda,\eps}=\log\br{(8-3(1-\lambda-\eps))/(3(1-\lambda-\eps))}$.
Assume $S\geq B_{\lambda,\eps}\sqrt d$.
Then, for every $T\geq3$,
    \begin{align*}
        \inf_\pi\sup_{\nu\in\cal C(d,K,S,\sigma_0^2,\lambda,\eps)}
        R_T(\pi;\nu)
        \geq
        \frac{\lambda(1-\lambda-\eps)e^{-1/6}}{16\sqrt T}.
    \end{align*}
\end{theorem}

\begin{remark}
The condition $S\geq B_{\lambda,\eps}\sqrt d$ is used to cover the full range $\lambda+\eps<1$.
When $\lambda+\eps$ is close to one, the hard instances require service probabilities close to one, which correspond to large logistic logits and therefore require a large parameter radius $S$.
\end{remark}

\subsection[Proof of the lower bound]{Proof of \cref{thm:lower_bound}}

We first introduce an explicit hard instance construction.
\begin{lemma}
\label{lem:lb_hard_instance}
Under the conditions of \Cref{thm:lower_bound}, fix $T\geq3$ and set
\begin{align*}
    \gamma=1-\lambda-\eps,\quad
    \Delta_T=\gamma/(8\sqrt T),\quad
    p_T= 1-\gamma/2+\Delta_T,\quad
    q_T= 1-\gamma/2-\Delta_T.
\end{align*}
Let $z_p=\log(p_T/(1-p_T))$ and $z_q=\log(q_T/(1-q_T))$.
Let the context set be
\begin{align*}
    \cal X_{\mathrm{lb}}
    =
    \{
        x(\sigma)=d^{-1/2}(1,\sigma_2,\ldots,\sigma_d)\tp:
        \sigma_j\in\{-1,+1\},\ j=2,\ldots,d
    \},
\end{align*}
and let the context distribution be uniform on $\cal X_{\mathrm{lb}}$.
Let $\m M=\diag(1,-1,1,1,\ldots,1)$, and define $\v\phi(x,1)= x$ and $\v\phi(x,a)=\m Mx$ for $a=2,\ldots,K$.
Let $e_j$ be the $j$-th 
unit vector, and define
\begin{align*}
    \v\theta_+^*
    =
    0.5\sqrt{d}\br{(z_p+z_q)e_1+(z_p-z_q)e_2},
    \qquad
    \v\theta_-^*
    =
    0.5\sqrt{d}\br{(z_p+z_q)e_1-(z_p-z_q)e_2}.
\end{align*}
Let $\nu_T^+$ and $\nu_T^-$ be the two instances with arrival rate $\lambda$, this context distribution and feature map, and unknown parameters $\v\theta_+^*$ and $\v\theta_-^*$, respectively.
Then $\nu_T^+,\nu_T^-\in\cal C(d,K,S,\sigma_0^2,\lambda,\eps)$.
\end{lemma}

The next proposition 
characterizes the optimal server sets and departure-probability properties of the hard instances.
\begin{proposition}
\label{prop:lb_hard_instance_properties}
For the hard instances in \Cref{lem:lb_hard_instance}, define, for each $x\in\cal X_{\mathrm{lb}}$, $\xi(x)=\sqrt d\,x_2\in\{-1,+1\}$ and
\begin{align*}
    \cal S_+(x)
    =
    \begin{cases}
        \cbr{1}, & \xi(x)=+1,\\
        \cbr{2,\ldots,K}, & \xi(x)=-1,
    \end{cases}
    \qquad
    \cal S_-(x)=[K]\setminus\cal S_+(x).
\end{align*}
For every $x\in\cal X_{\mathrm{lb}}$ and $s\in\{+,-\}$, servers in $\cal S_s(x)$ have departure probability $p_T$ under $\nu_T^s$, and servers outside $\cal S_s(x)$ have departure probability $q_T$ under $\nu_T^s$.
Moreover, $p_T>q_T>\lambda+\eps$, $p_T-q_T=(1-\lambda-\eps)/(4\sqrt T)$, and $u(1-u)\geq3(1-\lambda-\eps)/16$ for all $u\in[q_T,p_T]$.
\end{proposition}

Next, we state a pathwise comparison lemma for the hard instances.
\begin{lemma}
\label{lem:lb_oracle_domination}
Fix $s\in\{+,-\}$ and consider the hard instance $\nu_T^s$ in \Cref{lem:lb_hard_instance}.
For any deterministic policy, there exists a coupling under which the oracle queue is pathwise no larger than the queue generated by that policy at every round.
\end{lemma}

Finally, let $F=\cbr{A(T-2)=1}$.
The following proposition is a conditional consequence of the KL chain rule and the Bretagnolle--Huber inequality (\Cref{lem:app_bh}).
\begin{proposition}
\label{prop:lb_conditioned_testing}
Fix a deterministic policy $\pi$ and two instances $\nu^+,\nu^-$ with the same arrival process and context distribution.
Let $\cal H_{T-1}=\sigma\br{\v A(1),\v D(1),\ldots,\v A(T-2),\v D(T-2)}$.
Write $\P_+$ and $\P_-$ for probabilities under $\nu^+$ and $\nu^-$, respectively.
Suppose that, up to round $T-2$, the conditional departure distributions under the two instances are either identical or equal to $\Bern(p)$ and $\Bern(q)$ in some order.
If it holds that
$$
    (T-2)
    \max\left\{
        \KL\br{\Bern(p)\|\Bern(q)},
        \KL\br{\Bern(q)\|\Bern(p)}
    \right\}
    \leq
    1/6$$, 
then for every event $G\in\cal H_{T-1}$, we have 
   $$ \P_+(G^c\mid F)+\P_-(G\mid F)
    \geq
    0.5 e^{-1/6}.$$
\end{proposition}
Now, we are ready to start the proof of \Cref{thm:lower_bound}.

\begin{proof}[Proof of \Cref{thm:lower_bound}]
Fix a deterministic policy $\pi$.
Let $\nu_T^+,\nu_T^-$, $p_T$, and $q_T$ be as in \Cref{lem:lb_hard_instance}, and let $\cal S_+(\cdot)$ and $\cal S_-(\cdot)$ be as in \Cref{prop:lb_hard_instance_properties}.
For $s\in\{+,-\}$, write $\P_s$ and $\E_s$ for probability and expectation under $\nu_T^s$.
For each $s\in\{+,-\}$, define a comparison policy $\pi^s$ that agrees with $\pi$ in rounds $1,\ldots,T-2$ and, in round $T-1$, selects the same job as $\pi$ but chooses the smallest server in $\cal S_s(x)$ for that job context $x$.
Let $Q^s(T)$ and $D^s(T-1)$ be the queue length and departure random variable under $\pi^s$.
By \Cref{lem:lb_oracle_domination}, under $\nu_T^s$, $Q^*(T)\leq Q^s(T)$, and hence
\begin{align*}
    R_T(\pi;\nu_T^s)
    \geq
    \E_s\sqbr{Q^\pi(T)-Q^s(T)}.
\end{align*}
Recall that $F=\cbr{A(T-2)=1}$.
On $F$, the queue under $\pi$ is nonempty at the beginning of round $T-1$.
Since $\pi$ and $\pi^s$ have the same queue state at the beginning of round $T-1$ and the same arrival in round $T-1$,
\begin{align*}
    R_T(\pi;\nu_T^s)
    \geq
    \E_s\sqbr{
        \ind{F}
        \br{
            D^s(T-1)-D^\pi(T-1)
        }
    }.
\end{align*}
Let $(x_{T-1},a_{T-1})$ be the job-server pair selected by $\pi$ in round $T-1$, with an arbitrary fixed value if the queue is empty, and set
\begin{align*}
    G=\cbr{a_{T-1}\in\cal S_+(x_{T-1})}.
\end{align*}
On $F\cap G^c$, $\pi$ chooses a suboptimal server under $\nu_T^+$. On $F\cap G$, it chooses a suboptimal server under $\nu_T^-$.
Since $F$ depends only on the arrival process, $\P_+(F)=\P_-(F)=\lambda$.
By \Cref{prop:lb_hard_instance_properties}, under $\nu_T^+$, the conditional expectation of $D^+(T-1)-D^\pi(T-1)$ is $p_T-q_T$ on $F\cap G^c$ and zero on $F\cap G$.
Under $\nu_T^-$, the same conditional expectation is $p_T-q_T$ on $F\cap G$ and zero on $F\cap G^c$.
Thus we have
\begin{align}
    R_T(\pi;\nu_T^+)
    &\geq
    \lambda(p_T-q_T)\,\P_+(G^c\mid F),
    \qquad
    R_T(\pi;\nu_T^-)
    \geq
    \lambda(p_T-q_T)\,\P_-(G\mid F).
    \label{eq:proof_lower_regret}
\end{align}
Let $\cal H_{T-1}=\sigma\br{\v A(1),\v D(1),\ldots,\v A(T-2),\v D(T-2)}$.
Since $\pi$ is deterministic, $G\in\cal H_{T-1}$.
For each round up to $T-2$, if the queue is empty, the departure random variable is deterministically zero under both instances.
Otherwise, by \Cref{prop:lb_hard_instance_properties}, the conditional departure distributions under $\nu_T^+$ and $\nu_T^-$ are either identical or equal to $\Bern(p_T)$ and $\Bern(q_T)$ in some order.
By the Bernoulli KL bound (\Cref{lem:app_bernoulli_kl}) and \Cref{prop:lb_hard_instance_properties},
\begin{align*}
    &\max\cbr{
        \KL\br{\Bern(p_T)\|\Bern(q_T)},
        \KL\br{\Bern(q_T)\|\Bern(p_T)}
    }\\
    &\qquad\leq
    \frac{(p_T-q_T)^2}{2(3(1-\lambda-\eps)/16)}
    =
    \frac{1-\lambda-\eps}{6T}
    \leq
    \frac1{6T}.
\end{align*}
Thus the KL condition of \Cref{prop:lb_conditioned_testing} holds.
Applying \Cref{prop:lb_conditioned_testing} gives $\P_+(G^c\mid F)+\P_-(G\mid F)\geq e^{-1/6}/2$.
Substituting this bound into \Cref{eq:proof_lower_regret} gives
\begin{align*}
    R_T(\pi;\nu_T^+)+R_T(\pi;\nu_T^-)
    &\geq
    \lambda(p_T-q_T)
    \br{
        \P_+(G^c\mid F)+\P_-(G\mid F)
    }\\
    &\geq
    \frac{\lambda(p_T-q_T)e^{-1/6}}{2}
    =
    \frac{\lambda(1-\lambda-\eps)e^{-1/6}}{8\sqrt T},
\end{align*}
Therefore, $\max\cbr{R_T(\pi;\nu_T^+),R_T(\pi;\nu_T^-)}\geq\lambda(1-\lambda-\eps)e^{-1/6}/(16\sqrt T)$, which completes the proof.
For a randomized policy, conditioning on its internal random seed and averaging gives the same sum lower bound, and hence the same maximum lower bound.
\end{proof}

\section{Experiments} \label{sec:experiments}

We compare the empirical queue length regret of CQB-$\eta$-2 with standard baselines.
We use $T=2000$, $d=10$, $K=5$, $\lambda=0.7$, $\eps=0.1$, and $\kappa=50$.
The feature vectors and the unknown parameter are sampled coordinate-wise from $\Unif[-1,1]$, and instances that do not satisfy the slackness and $\kappa$ assumptions are rejected.
All curves are averaged over $10$ independent runs, and the error bands indicate $\pm 1$ standard deviation.
For the practical implementation, we use $\delta=T^{-1}$ and scale the confidence radius by $0.1$.
For CQB-$\eta$-2, we set the phase transition points to $\tau_1=0.2T$ and $\tau_2=0.3T$.
For CQB-$\varepsilon$, we set the initial random-exploration length to $\tau=0.2T$, which plays the same role as $\tau_1$ in CQB-$\eta$-2.
For ACQB, the random-exploration probability is $\eta(t)=\min\{1,c_1/\sqrt{t+1}\}$, and we set $c_1=20$.
The random policy selects both the job and the server uniformly at random, and FIFO+random selects the oldest job in the queue and a server uniformly at random.

\begin{figure}[t]
    \centering
    \includegraphics[width=\textwidth]{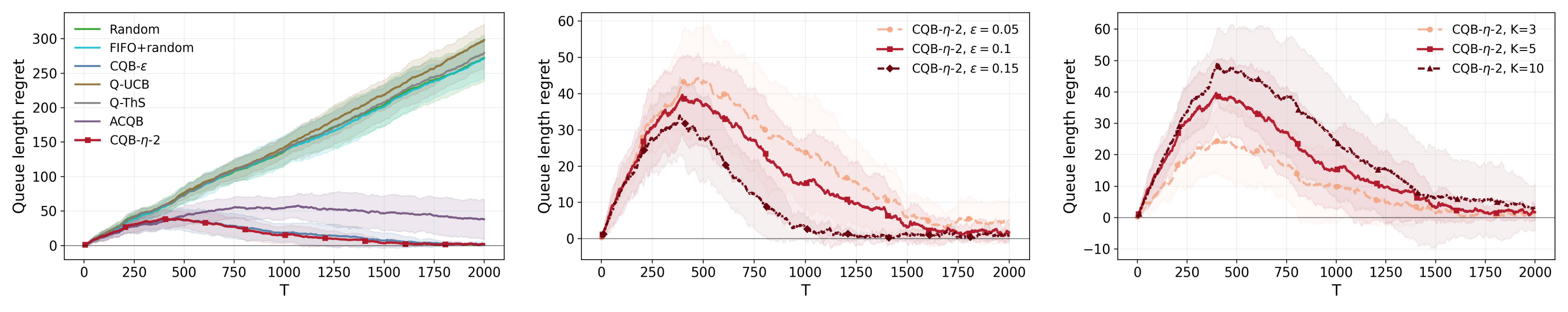}
    \caption{Queue length regret in synthetic experiments. Left: comparison with baselines. Middle: CQB-$\eta$-2 with $\eps\in\{0.05,0.1,0.15\}$. Right: CQB-$\eta$-2 with $K\in\{3,5,10\}$.}
    \label{fig:experiments}
\end{figure}

\Cref{fig:experiments} shows that CQB-$\eta$-2 and CQB-$\varepsilon$ have the smallest final regret in the baseline comparison.
The middle panel shows that larger slackness improves performance, and the right panel shows that CQB-$\eta$-2 remains stable across the tested values of $K$.
All experiments were conducted on a server equipped with an AMD EPYC 9354 32-Core Processor, 251 GiB of RAM, and one NVIDIA RTX A6000 GPU.

\section{Conclusion}

We studied contextual queueing bandits under stochastic contexts and showed that the queue length regret can be improved from the previous $\wtilde{\cal O}(T^{-1/4})$ rate to $\wtilde{\cal O}(T^{-1/2})$.
The main algorithmic idea is to stop random exploration after enough samples have been collected and use a pure UCB rule in the final phase.
We also proved an $\Omega(T^{-1/2})$ lower bound, showing that the dependence on $T$ is tight up to logarithmic factors.
An important direction for future work is to obtain lower bounds that also match the upper bound in the dependence on $d$ and $\kappa$.

\bibliographystyle{plainnat}
\bibliography{ref}

\clearpage
\appendix

\section{Related work}

\para{Queueing and contextual queueing bandits}
Queueing bandits study learning-while-scheduling problems in which unknown service rates must be learned while queue lengths are controlled \citep{krishnasamy2016regret,krishnasamy2021learning}.
This line of work includes queue length regret, dispatching under unknown service rates, decentralized queueing systems, MaxWeight-type learning algorithms, adversarial or nonstationary queueing models, and related queue-aware learning problems \citep{stahlbuhk2021learning,choudhury2021job,sentenac2021decentralized,freund2022efficient,yang2023learning,liang2018minimizing,huang2024lyapunov,krishnakumar2025minimizing,wijewardena2025bandit,gaitonde2020stability,hsu2022integrated}.
Classical queueing-control and scheduling works provide stability, drift, and routing tools that underlie these models \citep{lin1984optimal,neely2010stochastic,koole2002queueing,andrews2004scheduling,aksin2007modern,vilaplana2014queuing}.
Recent queueing and scheduling studies also consider modern service systems, including LLM inference and multi-LLM routing \citep{mitzenmacher2025queueing,yang2025queueingtheoreticperspectivelowlatency,fu2024efficient,lee2024design,jali2024efficient,bari2025optimal,murthy2024performance}.
Contextual queueing bandit variants incorporate job-specific features into departure models \citep{kim2024queueing,bae2026queue,bae2026learning}.
The closest prior works \citet{bae2026queue,bae2026learning} obtain $\wtilde{\cal O}(T^{-1/4})$ queue length regret under stochastic contexts, whereas we prove a $\wtilde{\cal O}(T^{-1/2})$ upper bound and an $\Omega(T^{-1/2})$ lower bound that match in $T$ up to logarithmic factors.

\para{Logistic bandits}
Our logistic departure model is related to generalized linear bandits, logistic bandits, and multinomial-logit bandits \citep{filippi2010parametric,abbasi2011improved,li2017provably,faury2020improved,jun2021improved,abeille2021instance,lee2024unified,bae2025neural,lee2025improved,zhang2025generalized,agrawal2019mnl,oh2019thompson,bae2026logistic}.
Related choice-model and routing bandits also study categorical feedback under logistic-type models \citep{chiang2025llm,wang2025mixllm,ong2024routellm,shirkavand2025cost,chen2024routerdc,zhuang2024embedllm}.
Unlike these works, contextual queueing bandits measure queue length regret, and the available job-server pairs depend on the queue state; this requires queue-specific arguments rather than a direct application of standard logistic-bandit regret analyses.

\section[Deferred proofs for section 4.1]{Deferred proofs for \cref{sec:upper_bound}}
\label{app:upper_bound_aux}

\subsection{Good events}
\label{app:upper_good_events}

Throughout this appendix, set $\delta=T^{-4}$ in the confidence radius.
For $t\in[T]$, define
\begin{align*}
    &\cal E_{\mathrm{pred}}(t)
    =
    \{\forall s\leq t, (x,a)\in\cal X\times[K], \\
    &\qquad\qquad\qquad |
        \mu(\langle\v\phi(x,a),\what{\v\theta}_{s-1}\rangle)
        -
        \mu(\langle\v\phi(x,a),\v\theta^*\rangle)
    |
    \leq
    \beta_{s-1}\norm{\v\phi(x,a)}_{\m V_{s-1}^{-1}}\}.
\end{align*}
Also define
\begin{align*}
    \cal E_{\mathrm{burn}}(t)
    =
    \{\forall s\in\{\tau_1+1,\ldots,t\}, (x,a)\in\cal X_s\times[K],~\beta_{s-1}\norm{\v\phi(x,a)}_{\m V_{s-1}^{-1}} \leq \eps/8\}.
\end{align*}
Let $\cal E_{\mathrm{drift}}(t)=\cal E_{\mathrm{pred}}(t)\cap\cal E_{\mathrm{burn}}(t)$.
We write $\cal E_{\mathrm{pred}}=\cal E_{\mathrm{pred}}(T)$, $\cal E_{\mathrm{burn}}=\cal E_{\mathrm{burn}}(T)$, and $\cal E_{\mathrm{drift}}=\cal E_{\mathrm{drift}}(T)$.
Then $\cal E_{\mathrm{pred}}(t)$, $\cal E_{\mathrm{burn}}(t)$, and $\cal E_{\mathrm{drift}}(t)$ are $\cal F_t$-measurable.
If $t_1\leq t_2$, then $\cal E_{\mathrm{drift}}(t_2)\subseteq\cal E_{\mathrm{drift}}(t_1)$.

Let
\begin{align*}
    N_{\mathrm{ph2}}
    =
    \sum_{s=\tau_1+1}^{\tau_2}
    \ind{A(s-1)=1,\ E(s-1)=1}
\end{align*}
and
\begin{align*}
    \m S_{\mathrm{ph2}}
    =
    \sum_{s=\tau_1+1}^{\tau_2}
    \ind{A(s-1)=1,\ E(s-1)=1}
    \v\phi_s\v\phi_s\tp.
\end{align*}
Define
\begin{align*}
    \cal E_{\mathrm{ph2}}
    =
    \cbr{
        \lambda_{\min}\br{\m S_{\mathrm{ph2}}}
        \geq
        \frac{\lambda\eps(\tau_2-\tau_1)\sigma_0^2}{16}
    }.
\end{align*}
Finally, let
\begin{align*}
    \nu_{\mathrm{ph2}}
    =
    \exp\br{-\frac{\lambda\eps(\tau_2-\tau_1)}{32}}
    +
    d\exp\br{-\frac{\lambda\eps(\tau_2-\tau_1)\sigma_0^2}{64}}.
\end{align*}

\begin{lemma}
\label{lem:app_good_event_prob}
With the confidence radius in \Cref{eq:beta} and the phase transition points in \Cref{eq:tau}, for sufficiently large absolute constants $c_1,c_2>0$, we have
\begin{align*}
    \P\br{\cal E_{\mathrm{pred}}^c}
    \leq
    T^{-4},
    \qquad
    \P\br{\cal E_{\mathrm{drift}}^c}
    \leq
    3T^{-4},
    \qquad
    \P\br{\cal E_{\mathrm{ph2}}^c}
    \leq
    \nu_{\mathrm{ph2}}.
\end{align*}
\end{lemma}

\begin{proof}
The first bound follows from \Cref{lem:app_import_prediction_error} with confidence level $\delta=T^{-4}$.
\Cref{lem:app_import_burnin}, with the same confidence level, gives $\P\br{\cal E_{\mathrm{burn}}^c}\leq2T^{-4}$.
Therefore, by the union bound,
\begin{align*}
    \P\br{\cal E_{\mathrm{drift}}^c}
    \leq
    \P\br{\cal E_{\mathrm{pred}}^c}
    +
    \P\br{\cal E_{\mathrm{burn}}^c}
    \leq
    3T^{-4}.
\end{align*}

For $s\in\{\tau_1+1,\ldots,\tau_2\}$, set $Z_s=\ind{A(s-1)=1,\ E(s-1)=1}$.
Let $\cal G_s$ be the sigma-field generated before $A(s-1)$ and $E(s-1)$ are drawn.
Since the Phase 2 random exploration probability is $\eps/4$, we have $\E\sqbr{Z_s\mid\cal G_s}=\lambda\eps/4$.
We apply \Cref{lem:app_adapted_chernoff} with $n=\tau_2-\tau_1$ and $p=\lambda\eps/4$.
This gives
\begin{align}
    \P\br{
        N_{\mathrm{ph2}}
        <
        \frac{\lambda\eps(\tau_2-\tau_1)}{8}
    }
    \leq
    \exp\br{-\frac{\lambda\eps(\tau_2-\tau_1)}{32}}.
    \label{eq:app_ph2_count}
\end{align}

Conditional on the random exploration times, the corresponding feature vectors are i.i.d. with the same distribution as $\v\phi(X,J)$, where $X\sim\cal D$ and $J\sim\Unif([K])$.
By \Cref{ass:iid}, $\E\sqbr{\v\phi(X,J)\v\phi(X,J)\tp}=\E_{X\sim\cal D}\sqbr{K^{-1}\sum_{a\in[K]}\v\phi(X,a)\v\phi(X,a)\tp}\succeq\sigma_0^2\I$.
Also, \Cref{ass:basic} gives $\norm{\v\phi(X,J)}_2\leq1$.
On the event $N_{\mathrm{ph2}}\geq\lambda\eps(\tau_2-\tau_1)/8$, there are at least $\lceil\lambda\eps(\tau_2-\tau_1)/8\rceil$ Phase 2 random exploration samples.
We apply \Cref{lem:app_matrix_chernoff} to the first $\lceil\lambda\eps(\tau_2-\tau_1)/8\rceil$ such samples in chronological order.
Adding positive semidefinite matrices cannot decrease the minimum eigenvalue, so $\lambda_{\min}\br{\m S_{\mathrm{ph2}}}\geq \lambda\eps(\tau_2-\tau_1)\sigma_0^2/16$ with conditional probability at least $1-d\exp\br{-\lambda\eps(\tau_2-\tau_1)\sigma_0^2/64}$.
Combining this bound with \Cref{eq:app_ph2_count} gives $\P\br{\cal E_{\mathrm{ph2}}^c}\leq\nu_{\mathrm{ph2}}$.
\end{proof}

\begin{lemma}
\label{lem:app_ucb_gap_burnin}
Fix $t>\tau_1$ and suppose that the queue is nonempty at the beginning of round $t$.
Let $(\bar x_t,\bar a_t)$ be a job-server pair selected by the UCB rule, that is,
\begin{align*}
    (\bar x_t,\bar a_t)
    \in
    \argmax_{x\in\cal X_t,\,a\in[K]}
    \cbr{
        \mu(\langle\v\phi(x,a),\what{\v\theta}_{t-1}\rangle)
        +
        \beta_{t-1}\norm{\v\phi(x,a)}_{\m V_{t-1}^{-1}}
    },
    \qquad
    \bar{\v\phi}_t
    =
    \v\phi(\bar x_t,\bar a_t).
\end{align*}
On $\cal E_{\mathrm{drift}}(t)$,
\begin{align*}
    \mu(\langle\v\phi_t^*,\v\theta^*\rangle)
    -
    \mu(\langle\bar{\v\phi}_t,\v\theta^*\rangle)
    \leq
    \frac{\eps}{4}.
\end{align*}
\end{lemma}

\begin{proof}
On $\cal E_{\mathrm{pred}}(t)$,
\begin{align*}
    \mu(\langle\v\phi_t^*,\v\theta^*\rangle)
    &\leq
    \mu(\langle\v\phi_t^*,\what{\v\theta}_{t-1}\rangle)
    +
    \beta_{t-1}\norm{\v\phi_t^*}_{\m V_{t-1}^{-1}}\\
    &\leq
    \mu(\langle\bar{\v\phi}_t,\what{\v\theta}_{t-1}\rangle)
    +
    \beta_{t-1}\norm{\bar{\v\phi}_t}_{\m V_{t-1}^{-1}}\\
    &\leq
    \mu(\langle\bar{\v\phi}_t,\v\theta^*\rangle)
    +
    2\beta_{t-1}\norm{\bar{\v\phi}_t}_{\m V_{t-1}^{-1}}.
\end{align*}
The second inequality follows from the definition of the UCB rule.
On $\cal E_{\mathrm{burn}}(t)$, the last term is at most $\eps/4$.
\end{proof}

\subsection[Proof of the late departure gap lemma]{Proof of \cref{lem:late_departure_gap}}
\label{app:proof_late_departure_gap}

\begin{proof}[Proof of \Cref{lem:late_departure_gap}]
Let
\begin{align*}
    g_t
    =
    \mu(\langle\v\phi_t^*,\v\theta^*\rangle)
    -
    \mu(\langle\v\phi_t,\v\theta^*\rangle).
\end{align*}
If the queue is empty in round $t$, we set $g_t=0$.
Hence it suffices to consider the case where the queue is nonempty.

On $\cal E_{\mathrm{ph2}}$, the matrix $\m V_{\tau_2}$ dominates the Phase 2 random-exploration design matrix, and therefore
\begin{align*}
    \lambda_{\min}(\m V_{\tau_2})
    \geq
    \frac{\lambda\eps(\tau_2-\tau_1)\sigma_0^2}{16}.
\end{align*}
For every $t\geq\tau_2+1$, we have $\m V_{t-1}\succeq\m V_{\tau_2}$, and hence, using $\norm{\v\phi_t}_2\leq1$,
\begin{align*}
    \norm{\v\phi_t}_{\m V_{t-1}^{-1}}^2
    \leq
    \frac{16}{\lambda\eps(\tau_2-\tau_1)\sigma_0^2}.
\end{align*}
Since $\tau_2-\tau_1\geq T/2$, we have
\begin{align}
    \norm{\v\phi_t}_{\m V_{t-1}^{-1}}^2
    \leq
    \frac{32}{\lambda\eps\sigma_0^2T}.
    \label{eq:app_late_uncertainty_T}
\end{align}

Now consider a Phase 3 round $t\in\{\tau_2+1,\ldots,T-1\}$.
By construction, Phase 3 uses the pure UCB rule.
On $\cal E_{\mathrm{pred}}$, for the optimal pair $(x_t^*,a_t^*)$ in the current queue and the pair $(x_t,a_t)$ selected by the algorithm,
\begin{align*}
    \mu(\langle\v\phi_t^*,\v\theta^*\rangle)
    &\leq
    \mu(\langle\v\phi_t^*,\what{\v\theta}_{t-1}\rangle)
    +
    \beta_{t-1}\norm{\v\phi_t^*}_{\m V_{t-1}^{-1}}\\
    &\leq
    \mu(\langle\v\phi_t,\what{\v\theta}_{t-1}\rangle)
    +
    \beta_{t-1}\norm{\v\phi_t}_{\m V_{t-1}^{-1}}\\
    &\leq
    \mu(\langle\v\phi_t,\v\theta^*\rangle)
    +
    2\beta_{t-1}\norm{\v\phi_t}_{\m V_{t-1}^{-1}}.
\end{align*}
The second inequality follows from the definition of the UCB rule.
Therefore, on $\cal E_{\mathrm{pred}}\cap\cal E_{\mathrm{ph2}}$,
\begin{align*}
    g_t
    \leq
    2\beta_{t-1}\norm{\v\phi_t}_{\m V_{t-1}^{-1}}
    \leq
    2\beta_T\norm{\v\phi_t}_{\m V_{t-1}^{-1}}.
\end{align*}
Using \Cref{eq:app_late_uncertainty_T}, on $\cal E_{\mathrm{pred}}\cap\cal E_{\mathrm{ph2}}$,
\begin{align*}
    g_t^2
    \leq
    \frac{128\beta_T^2}{\lambda\eps\sigma_0^2T}.
\end{align*}

Finally, since $0\leq g_t\leq1$, \Cref{lem:app_good_event_prob} gives
\begin{align*}
    \E\sqbr{g_t^2}
    &\leq
    \E\sqbr{
        g_t^2\ind{\cal E_{\mathrm{pred}}\cap\cal E_{\mathrm{ph2}}}
    }
    +
    \P\br{\cal E_{\mathrm{pred}}^c}
    +
    \P\br{\cal E_{\mathrm{ph2}}^c}\\
    &\leq
    \frac{128\beta_T^2}{\lambda\eps\sigma_0^2T}
    +
    T^{-4}
    +
    \nu_{\mathrm{ph2}}.
\end{align*}
Taking square roots gives the desired result.
\end{proof}

\subsection[Proof of the effect gap lemma]{Proof of \cref{lem:effect_gap}}
\label{app:proof_lemma_effect_gap}

\begin{proof}[Proof of \Cref{lem:effect_gap}]
We first record the negative drift property after Phase 1 for the coupled process governed by $\pi_{T-1}$.
This process has the same marginal law as the original queue governed by the learning policy.
Fix $t>\tau_1$ and suppose that $Q_{T-1}(t)>0$.
Let $\eta_t$ be an upper bound on the conditional probability that the random exploration rule is used in round $t$; in Phase 2, $\eta_t=\eps/4$, and in Phase 3, $\eta_t=0$.
Let $(x_t^*,a_t^*)$ be an optimal job-server pair in the queue of this process at the beginning of round $t$.
By \Cref{ass:slackness},
\begin{align*}
    \mu(\langle\v\phi_t^*,\v\theta^*\rangle)
    \geq
    \lambda+\eps.
\end{align*}
If the algorithm uses random exploration, the departure-rate loss relative to the optimal pair is at most one.
If it uses the UCB rule, the loss is at most $\eps/4$ on $\cal E_{\mathrm{drift}}(t)$ by \Cref{lem:app_ucb_gap_burnin}.
Therefore, on $\cal E_{\mathrm{drift}}(t)\cap\cbr{Q_{T-1}(t)>0}$,
\begin{align*}
    \E\sqbr{D_{T-1}(t)\mid\cal F_t}
    &\geq
    \mu(\langle\v\phi_t^*,\v\theta^*\rangle)
    -
    \eta_t
    -
    \frac{\eps}{4}\\
    &\geq
    \lambda+\eps-\frac{\eps}{4}-\frac{\eps}{4}
    =
    \lambda+\frac{\eps}{2}.
\end{align*}
Hence, whenever $Q_{T-1}(t)>0$ and $t>\tau_1$, on $\cal E_{\mathrm{drift}}(t)$,
\begin{align}
    \E\sqbr{A(t)-D_{T-1}(t)\mid\cal F_t}
    \leq
    -\frac{\eps}{2}.
    \label{eq:app_effect_policy_negative_drift}
\end{align}

We use this drift condition to obtain a tail bound for $Q_{T-1}(t)$.
Let $\zeta=\eps/2$, $\rho=\exp(-\eps^2/8)$, and $b=\exp(\zeta)$.
Define the deterministic bad-round count $B(t)=\min\{t,\tau_1\}$.
Thus rounds $1,\ldots,\tau_1$ are treated as bad rounds, and all later rounds have negative drift on $\cal E_{\mathrm{drift}}(t)$.
Define
\begin{align*}
    V(t)
    =
    \br{\frac{b}{\rho}}^{-B(t-1)}
    \exp\br{\zeta Q_{T-1}(t)}.
\end{align*}
We claim that, for every $t\in[T-1]$,
\begin{align}
    \E\sqbr{\ind{\cal E_{\mathrm{drift}}(t+1)}V(t+1)}
    \leq
    \rho
    \E\sqbr{\ind{\cal E_{\mathrm{drift}}(t)}V(t)}
    +
    1.
    \label{eq:app_effect_weighted_recursion}
\end{align}

To prove this, first consider a round $t>\tau_1$.
If $Q_{T-1}(t)>0$, then, on $\cal E_{\mathrm{drift}}(t)$, \Cref{eq:app_effect_policy_negative_drift} and \Cref{lem:app_cond_hoeffding} give
\begin{align*}
    \E\sqbr{\exp\br{\zeta(A(t)-D_{T-1}(t))}\mid\cal F_t}
    \leq
    \exp\br{-\zeta\eps/2+\zeta^2/2}
    =
    \exp(-\eps^2/8)
    =
    \rho.
\end{align*}
If $Q_{T-1}(t)=0$, then the actual departure from the queue is zero and $Q_{T-1}(t+1)=A(t)$, so
\begin{align*}
    \E\sqbr{\exp\br{\zeta Q_{T-1}(t+1)}\mid\cal F_t}
    \leq
    \exp(\zeta)
    \leq
    1+\rho,
\end{align*}
where the last inequality holds for $\eps\in(0,1)$.
Combining the two cases, for every $t>\tau_1$, on $\cal E_{\mathrm{drift}}(t)$,
\begin{align}
    \E\sqbr{\exp\br{\zeta Q_{T-1}(t+1)}\mid\cal F_t}
    \leq
    1+\rho\exp\br{\zeta Q_{T-1}(t)}.
    \label{eq:app_effect_good_round_mgf}
\end{align}
For $t\leq\tau_1$, we use the crude bound
\begin{align}
    \E\sqbr{\exp\br{\zeta Q_{T-1}(t+1)}\mid\cal F_t}
    \leq
    1+b\exp\br{\zeta Q_{T-1}(t)}.
    \label{eq:app_effect_bad_round_mgf}
\end{align}
Since $\cal E_{\mathrm{drift}}(t+1)\subseteq\cal E_{\mathrm{drift}}(t)$ and $\cal E_{\mathrm{drift}}(t)\in\cal F_t$, \Cref{eq:app_effect_good_round_mgf,eq:app_effect_bad_round_mgf} and the definition of $V(t)$ give \Cref{eq:app_effect_weighted_recursion}.
Solving the recursion and using $Q_{T-1}(1)=0$,
\begin{align}
    \E\sqbr{\ind{\cal E_{\mathrm{drift}}(t)}V(t)}
    \leq
    1+\frac{1}{1-\rho}
    \leq
    17\eps^{-2},
    \label{eq:app_effect_weighted_bound}
\end{align}
where we used $1-e^{-x}\geq x/2$ for $x\in[0,1]$.

Let $a=1+\eps/4$.
Since $a=\zeta^{-1}\log(b/\rho)$, $\cal E_{\mathrm{drift}}\subseteq\cal E_{\mathrm{drift}}(t)$, and $B(t-1)\leq\tau_1$, Markov's inequality and \Cref{eq:app_effect_weighted_bound} imply that, for every $t$ and every $y\geq0$,
\begin{align}
    \P\br{
        Q_{T-1}(t)\geq a\tau_1+y,\ \cal E_{\mathrm{drift}}
    }
    \leq
    17\eps^{-2}\exp(-\zeta y).
    \label{eq:app_effect_queue_tail}
\end{align}

We next bound the probability that a one-job discrepancy created in a fixed round $i$ survives until round $T$.
Let $n= T-i-1$.
If $n<4\tau_1/\eps$, then the trivial bound $\wtilde\psi_i(T)\leq1$ gives the first part of the lemma.
Thus assume $n\geq4\tau_1/\eps$.

Condition on $\cal F_i^+$ and on the disagreement event $\v D_i(i)=0$, $\v D_{i-1}(i)=1$.
Let $\bar Q_i(i+1)$ be the queue length of the process governed by $\pi_i$ at the beginning of round $i+1$ under this disagreement event.
After round $i$, the process governed by $\pi_i$ follows the oracle policy.
Let $H$ be the first time after round $i$ at which this process hits queue length zero.

By \Cref{lem:app_refined_coupling}, under the disagreement event, $\psi_i(T)\in\{0,1\}$, and $\psi_i(T)=1$ implies $H>T$.
Therefore,
\begin{align*}
    \wtilde\psi_i(T)
    \leq
    \P\br{
        H>T
        \mid
        \cal F_i^+,\v D_i(i)=0,\v D_{i-1}(i)=1
    }.
\end{align*}

We now bound this survival probability.
Suppose that $\bar Q_i(i+1)=q$.
For rounds $t=i+1,\ldots,T-1$, as long as the queue is nonempty, the oracle policy has departure probability at least $\lambda+\eps$.
Thus, for such rounds,
\begin{align*}
    \E\sqbr{A(t)-D_i(t)\mid\cal G_t}
    \leq
    -\eps,
\end{align*}
where $\cal G_t$ is the history of the process governed by $\pi_i$ up to the beginning of round $t$.
By \Cref{lem:app_cond_hoeffding} with $\zeta=\eps/2$,
\begin{align*}
    \E\sqbr{\exp\br{\zeta(A(t)-D_i(t))}\mid\cal G_t}
    \leq
    \exp\br{-\zeta\eps+\zeta^2/2}
    \leq
    \exp(-\eps^2/4).
\end{align*}

We apply this bound only before the hitting time.
In the next display, all probabilities and expectations are conditional on $\bar Q_i(i+1)=q$, $\cal F_i^+$, and the disagreement event $\v D_i(i)=0,\v D_{i-1}(i)=1$.
For $m\in\{i+1,\ldots,T\}$, let $\cal H_m=\cbr{H>m}$.
For $m\in\{i+1,\ldots,T-1\}$, on $\cal H_m$, the queue is nonempty at the beginning of round $m$, and hence the following one-step stopped inequality holds:
\begin{align*}
    &\E\sqbr{
        \ind{\cal H_{m+1}}
        \exp\br{
            \zeta
            \sum_{s=i+1}^{m}
            \br{A(s)-D_i(s)}
        }
        \mid
        \cal G_m
    }\\
    &\qquad
    \leq
    \ind{\cal H_m}
    \exp\br{
        \zeta
        \sum_{s=i+1}^{m-1}
        \br{A(s)-D_i(s)}
    }
    \E\sqbr{\exp\br{\zeta(A(m)-D_i(m))}\mid\cal G_m}\\
    &\qquad
    \leq
    \exp(-\eps^2/4)
    \ind{\cal H_m}
    \exp\br{
        \zeta
        \sum_{s=i+1}^{m-1}
        \br{A(s)-D_i(s)}
    }.
\end{align*}
Iterating the last display gives
\begin{align*}
    \E\sqbr{
        \ind{\cal H_T}
        \exp\br{
            \zeta
            \sum_{s=i+1}^{T-1}
            \br{A(s)-D_i(s)}
        }
    }
    \leq
    \exp\br{-\frac{\eps^2}{4}n}.
\end{align*}
On $\cal H_T$, we have
\begin{align*}
    q+\sum_{s=i+1}^{T-1}\br{A(s)-D_i(s)}\geq1.
\end{align*}
Therefore,
\begin{align}
    \P\br{
        H>T
        \mid
        \bar Q_i(i+1)=q,\cal F_i^+,\v D_i(i)=0,\v D_{i-1}(i)=1
    }
    \leq
    \min\cbr{
        1,
        \exp\br{
            \zeta(q-1)
            -
            \frac{\eps^2}{4}n
        }
    }.
    \label{eq:app_effect_survival_bound}
\end{align}
This stopped estimate handles the empty-queue case: no negative drift is asserted after the queue becomes empty.

It remains to average \Cref{eq:app_effect_survival_bound} over $\cal F_i^+$.
Set
\begin{align*}
    \omega
    =
    \frac{4\tau_1}{\eps},
    \qquad
    y_i
    =
    \frac{\eps}{4}(n-\omega).
\end{align*}
Since $n\geq\omega$, we have $y_i\geq0$.
Also, under the disagreement event, $\bar Q_i(i+1)\leq Q_{T-1}(i)+1$.
Hence \Cref{eq:app_effect_queue_tail} gives
\begin{align}
    \P\br{
        \bar Q_i(i+1)\geq a\tau_1+1+y_i,\ \cal E_{\mathrm{drift}}
    }
    \leq
    17\eps^{-2}\exp(-\zeta y_i).
    \label{eq:app_effect_barQ_tail}
\end{align}
On the complementary event $\bar Q_i(i+1)<a\tau_1+1+y_i$, \Cref{eq:app_effect_survival_bound} yields
\begin{align*}
    \P\br{
        H>T
        \mid
        \cal F_i^+,\v D_i(i)=0,\v D_{i-1}(i)=1
    }
    &\leq
    \exp\br{
        \zeta(a\tau_1+y_i)
        -
        \frac{\eps^2}{4}n
    }\\
    &\leq
    \exp\br{
        -\frac{\eps^2}{32}(n-\omega)
    }.
\end{align*}
The last inequality follows from $\zeta=\eps/2$, $a=1+\eps/4$, $y_i=\eps(n-\omega)/4$, $\omega=4\tau_1/\eps$, and $\eps\in(0,1)$.

Combining the bound on the event $\bar Q_i(i+1)<a\tau_1+1+y_i$ with \Cref{eq:app_effect_barQ_tail} and \Cref{lem:app_good_event_prob},
\begin{align*}
    \E\sqbr{\wtilde\psi_i(T)}
    &\leq
    \P\br{\cal E_{\mathrm{drift}}^c}
    +
    \exp\br{
        -\frac{\eps^2}{32}(n-\omega)
    }
    +
    17\eps^{-2}\exp(-\zeta y_i)\\
    &\leq
    3T^{-4}
    +
    19\eps^{-2}
    \exp\br{
        -\frac{\eps^2}{32}
        \br{T-i-1-\frac{4\tau_1}{\eps}}
    },
\end{align*}
where we used $\zeta y_i=\eps^2(n-\omega)/8$ and $\eps^{-2}\geq1$.
Taking square roots, using the trivial bound $\wtilde\psi_i(T)\leq1$, and relabeling $i$ as $t$ complete the proof.
\end{proof}

\section[Deferred proofs for section 4.2]{Deferred proofs for \cref{sec:lower_bound}}
\label{app:lower_bound_aux}

\subsection[Proof of hard instance lemma]{Proof of \cref{lem:lb_hard_instance}}

\begin{proof}[Proof of \Cref{lem:lb_hard_instance}]
Let $\gamma=1-\lambda-\eps$, $\Delta_T=\gamma/(8\sqrt T)$, $p_T=1-\gamma/2+\Delta_T$, and $q_T=1-\gamma/2-\Delta_T$.
Since $T\geq3$, we have $p_T\leq1-3\gamma/8<1$, $q_T\geq1-5\gamma/8\geq3/8$, and $q_T-(\lambda+\eps)\geq3\gamma/8>0$.
Thus $z_p$ and $z_q$ are well-defined, and $p_T>\lambda+\eps$.

It remains to verify that $\nu_T^+,\nu_T^-\in\cal C(d,K,S,\sigma_0^2,\lambda,\eps)$.
Since $\norm{x}_2=1$ and $\m M$ is orthogonal, $\norm{\v\phi(x,a)}_2=1$.
Also, if $X$ is uniform on $\cal X_{\mathrm{lb}}$, then $\E[XX\tp]=d^{-1}\I$ and
\begin{align*}
    \E\sqbr{
        \frac1K\sum_{a\in[K]}\v\phi(X,a)\v\phi(X,a)\tp
    }
    =
    \frac1d \I.
\end{align*}
Thus \Cref{ass:iid} holds for every $\sigma_0^2\leq1/d$.
For every $x\in\cal X_{\mathrm{lb}}$ and each sign $s\in\{+,-\}$, the logits $\langle\v\phi(x,a),\v\theta_s^*\rangle$ over $a\in[K]$ take values in $\cbr{z_p,z_q}$ and include $z_p$.
Therefore $\max_{a\in[K]}\mu(\langle\v\phi(x,a),\v\theta_s^*\rangle)=p_T>\lambda+\eps$, and the traffic slackness condition holds.
For the parameter radius, let $L_\gamma=\log\br{(8-3\gamma)/(3\gamma)}$.
The bounds $q_T\geq3/8$ and $p_T\leq1-3\gamma/8$ imply $\abs{z_p},\abs{z_q}\leq L_\gamma$.
Therefore
\begin{align*}
    \norm{\v\theta_\pm^*}_2
    =
    \sqrt{\frac d2\br{z_p^2+z_q^2}}
    \leq
    L_\gamma\sqrt d
    =
    B_{\lambda,\eps}\sqrt d
    \leq
    S.
\end{align*}
Finally, \Cref{ass:constants} holds for finite constants because all logits over $\Theta$ lie in $\sqbr{-S,S}$ and $\dot\mu$ is positive and continuous on this interval.
Hence $\nu_T^+,\nu_T^-\in\cal C(d,K,S,\sigma_0^2,\lambda,\eps)$.
\end{proof}

\subsection[Proof of hard instance properties]{Proof of \cref{prop:lb_hard_instance_properties}}

\begin{proof}[Proof of \Cref{prop:lb_hard_instance_properties}]
Let $\gamma=1-\lambda-\eps$.
For $x\in\cal X_{\mathrm{lb}}$, let $\xi(x)=\sqrt d\,x_2\in\{-1,+1\}$.
If $\xi(x)=+1$, then
\begin{align*}
    \langle\v\phi(x,1),\v\theta_+^*\rangle=z_p,
    \qquad
    \langle\v\phi(x,a),\v\theta_+^*\rangle=z_q
    \quad\text{for all }a\in\{2,\ldots,K\}.
\end{align*}
If $\xi(x)=-1$, then these two logits are reversed.
Thus, under $\nu_T^+$, servers in $\cal S_+(x)$ have departure probability $p_T$, and servers outside $\cal S_+(x)$ have departure probability $q_T$.
The same calculation with $\v\theta_-^*$ gives the corresponding statement under $\nu_T^-$.

It remains to verify the numerical bounds.
Since $T\geq3$, we have $p_T\leq1-3\gamma/8<1$, $q_T\geq1-5\gamma/8\geq3/8$, and $q_T-(\lambda+\eps)\geq3\gamma/8>0$.
Thus $p_T>q_T>\lambda+\eps$ and $p_T-q_T=\gamma/(4\sqrt T)$.
Also, $p_T\geq1/2$, $1-p_T\geq3\gamma/8$, $q_T\geq3/8$, and $1-q_T\geq\gamma/2$.
Thus
\begin{align*}
    p_T(1-p_T)\geq\frac{3\gamma}{16},
    \qquad
    q_T(1-q_T)\geq\frac{3\gamma}{16}.
\end{align*}
For every $u\in[q_T,p_T]$, the concavity of $u(1-u)$ gives
\begin{align*}
    u(1-u)
    \geq
    \min\cbr{p_T(1-p_T),q_T(1-q_T)}
    \geq
    \frac{3\gamma}{16}.
\end{align*}
Substituting $\gamma=1-\lambda-\eps$ completes the proof.
\end{proof}

\subsection[Proof of oracle domination lemma]{Proof of \cref{lem:lb_oracle_domination}}

\begin{proof}[Proof of \Cref{lem:lb_oracle_domination}]
Fix $s\in\{+,-\}$ and a deterministic policy $\rho$.
Let $Q^*(t)$ be the oracle queue under $\nu_T^s$, and let $Q^\rho(t)$ be the queue generated by $\rho$ under the same instance.
Couple the two processes by using the same arrival random variables and the same uniform random variables to generate departures.
We prove the claim by induction.
At $t=1$, both queues are empty.
Assume $Q^*(t)\leq Q^\rho(t)$.
If $Q^*(t)=0$, then $Q^*(t+1)=A(t)$, while $D^\rho(t)\leq Q^\rho(t)$ implies
\begin{align*}
    Q^\rho(t+1)=Q^\rho(t)+A(t)-D^\rho(t)\geq A(t).
\end{align*}
If $Q^*(t)>0$, then $Q^\rho(t)>0$ as well.
By \Cref{prop:lb_hard_instance_properties}, the oracle can choose a job-server pair with departure probability $p_T$, and every feasible job-server pair has departure probability at most $p_T$.
Let $r_t^\rho\leq p_T$ be the departure probability of the job-server pair selected by $\rho$.
Using the same uniform random variable $U_t$ for the departures, $D^*(t)=\ind{U_t\leq p_T}$ and $D^\rho(t)=\ind{U_t\leq r_t^\rho}$, so $D^*(t)\geq D^\rho(t)$.
Thus
\begin{align*}
    Q^*(t+1)
    =
    Q^*(t)+A(t)-D^*(t)
    \leq
    Q^\rho(t)+A(t)-D^\rho(t)
    =
    Q^\rho(t+1).
\end{align*}
\end{proof}

\subsection[Proof of conditional testing proposition]{Proof of \cref{prop:lb_conditioned_testing}}

\begin{proof}[Proof of \Cref{prop:lb_conditioned_testing}]
Let $P_+^F$ and $P_-^F$ be the conditional laws on $\cal H_{T-1}$ under $\nu^+$ and $\nu^-$ given $F$.
Since the initial queue is empty, the observations generating $\cal H_{T-1}$ determine the queue state before each round up to $T-1$.
For a deterministic policy, the selected job-server pair is therefore determined by the past history.
The two instances have the same arrival process and context distribution, and conditioning on $F=\cbr{A(T-2)=1}$ fixes only the shared arrival random variable $A(T-2)$.
Thus, in each round, only the conditional distribution of the departure random variable can differ between the two instances.
If the queue is empty, the departure random variable is deterministically zero under both instances.
Otherwise, by assumption, the two conditional distributions of the departure random variable are either identical or are $\Bern(p)$ and $\Bern(q)$ in some order.
Therefore, the KL chain rule bounds $\KL\br{P_+^F\|P_-^F}$ by the sum of the one-round conditional KL divergences.
The arrival and context terms contribute zero, and each departure term is at most the larger KL divergence between $\Bern(p)$ and $\Bern(q)$.
Thus,
\begin{align*}
    \KL\br{P_+^F\|P_-^F}
    \leq
    (T-2)
    \max\cbr{
        \KL\br{\Bern(p)\|\Bern(q)},
        \KL\br{\Bern(q)\|\Bern(p)}
    }.
\end{align*}
By assumption, the right-hand side is at most $1/6$.
By the Bretagnolle--Huber inequality (\Cref{lem:app_bh}) applied to $P_+^F$, $P_-^F$, and $G$,
\begin{align*}
    \P_+(G^c\mid F)+\P_-(G\mid F)
    =
    P_+^F(G^c)+P_-^F(G)
    \geq
    \frac12e^{-1/6}.
\end{align*}
\end{proof}

\section{Auxiliary lemmas}
\label{app:auxiliary_lemmas}

This section collects auxiliary results used in the appendix.
The following prediction-error bound is Lemma 28 of \citet{bae2026logistic}, stated in the notation of this paper.
\begin{lemma}
\label{lem:app_import_prediction_error}
With the confidence radius defined by \Cref{eq:beta}, it holds with probability at least $1-\delta$ that
\begin{align*}
    \abs{
        \mu(\langle\v\phi(x,a),\what{\v\theta}_{s-1}\rangle)
        -
        \mu(\langle\v\phi(x,a),\v\theta^*\rangle)
    }
    \leq
    \beta_{s-1}\norm{\v\phi(x,a)}_{\m V_{s-1}^{-1}}
\end{align*}
for all $s\in[T]$ and all $(x,a)\in\cal X\times[K]$.
\end{lemma}

The following design-matrix lower bound is Proposition 1 of \citet{li2017provably}.
\begin{lemma}
\label{lem:app_li_design}
Let $Y_1,\ldots,Y_n$ be i.i.d. random vectors drawn from a distribution supported on the unit ball in $\R^d$.
Let $\Sigma=\E[Y_1Y_1\tp]$.
For any $B>0$ and $\delta>0$, there exist absolute constants $C_1,C_2>0$ such that
\begin{align*}
    \lambda_{\min}\br{\sum_{s=1}^n Y_sY_s\tp}
    \geq
    B
\end{align*}
with probability at least $1-\delta$, provided that
\begin{align*}
    n
    \geq
    \br{
        \frac{C_1\sqrt d+C_2\sqrt{\log(1/\delta)}}
        {\lambda_{\min}(\Sigma)}
    }^2
    +
    \frac{2B}{\lambda_{\min}(\Sigma)}.
\end{align*}
\end{lemma}

The following burn-in uncertainty bound is adapted from Lemma 5.2 of \citet{bae2026queue}.
\begin{lemma}
\label{lem:app_import_burnin}
Suppose that $\tau_1$ is chosen according to \Cref{eq:tau} with sufficiently large absolute constants $c_1,c_2>0$.
Then, with probability at least $1-2\delta$,
\begin{align*}
    \beta_{t-1}\norm{\v\phi(x,a)}_{\m V_{t-1}^{-1}}
    \leq
    \eps/8
\end{align*}
for all $t\in\{\tau_1+1,\ldots,T\}$ and all $(x,a)\in\cal X_t\times[K]$.
\end{lemma}

\begin{proof}[Proof of \Cref{lem:app_import_burnin}]
Let $B_1=64\beta_T^2/\eps^2$, and let $N_{\mathrm{ph1}}=\sum_{s=1}^{\tau_1}\ind{A(s-1)=1}$ be the number of Phase 1 random exploration samples.
During Phase 1, every arrival is used for random exploration, and the server is selected uniformly at random.
Conditional on the random exploration times, the corresponding feature vectors are i.i.d. with the same distribution as $\v\phi(X,J)$, where $X\sim\cal D$ and $J\sim\Unif([K])$.
By \Cref{ass:iid}, $\E\sqbr{\v\phi(X,J)\v\phi(X,J)\tp}\succeq\sigma_0^2\I$.
Set
\begin{align*}
    n_{\mathrm{ph1}}
    =
    \br{
        \frac{C_1\sqrt d+C_2\sqrt{\log(1/\delta)}}
        {\sigma_0^2}
    }^2
    +
    \frac{2B_1}{\sigma_0^2}.
\end{align*}
By taking the absolute constants $c_1,c_2$ in the definition of $\tau_1$ sufficiently large,
\begin{align*}
    \frac{\lambda\tau_1}{2}
    \geq
    n_{\mathrm{ph1}},
    \qquad
    \exp\br{-\frac{\lambda\tau_1}{8}}
    \leq
    \delta.
\end{align*}
For each $s\in[\tau_1]$, let $\cal G_s$ be the sigma-field generated by the history before $A(s-1)$ is drawn.
Then $\ind{A(s-1)=1}$ is $\{0,1\}$-valued and has conditional mean $\lambda$ given $\cal G_s$.
We apply \Cref{lem:app_adapted_chernoff} with $n=\tau_1$ and $p=\lambda$ to get
\begin{align*}
    \P\br{
        N_{\mathrm{ph1}}
        <
        n_{\mathrm{ph1}}
    }
    \leq
    \P\br{
        N_{\mathrm{ph1}}
        <
        \frac{\lambda\tau_1}{2}
    }
    \leq
    \delta.
\end{align*}
On the event $N_{\mathrm{ph1}}\geq n_{\mathrm{ph1}}$, we apply \Cref{lem:app_li_design} to the first $\lceil n_{\mathrm{ph1}}\rceil$ Phase 1 random exploration samples in chronological order, with $n=\lceil n_{\mathrm{ph1}}\rceil$, $B=B_1$, and $\lambda_{\min}(\Sigma)=\sigma_0^2$.
Adding positive semidefinite matrices cannot decrease the minimum eigenvalue, so $\lambda_{\min}(\m S_{\mathrm{ph1}})\geq B_1$ with conditional probability at least $1-\delta$, where
\begin{align*}
    \m S_{\mathrm{ph1}}
    =
    \sum_{s=1}^{\tau_1}\ind{A(s-1)=1}\v\phi_s\v\phi_s\tp
\end{align*}
is the design matrix formed by the Phase 1 random exploration samples.
By the union bound, this design lower bound holds with probability at least $1-2\delta$.

On this event, for every $t\in\{\tau_1+1,\ldots,T\}$,
\begin{align*}
    \m V_{t-1}
    \succeq
    \m V_{\tau_1}
    \succeq
    \m S_{\mathrm{ph1}}.
\end{align*}
Using $\norm{\v\phi(x,a)}_2\leq1$ and $\beta_{t-1}\leq\beta_T$,
\begin{align*}
    \beta_{t-1}\norm{\v\phi(x,a)}_{\m V_{t-1}^{-1}}
    \leq
    \frac{\beta_T}{\sqrt{B_1}}
    =
    \frac{\eps}{8}
\end{align*}
for all $(x,a)\in\cal X_t\times[K]$.
\end{proof}

The following coupling property follows from Lemmas B.1 and C.5 of \citet{bae2026queue} and Lemma E.1 of \citet{bae2026learning}, stated in the notation of this paper.
\begin{lemma}
\label{lem:app_refined_coupling}
Consider the coupled consecutive policy-switching queues governed by $\pi_t$ and $\pi_{t-1}$.
If $\v D_t(t)=\v D_{t-1}(t)$, then $\psi_t(T)\in\{-1,0\}$.
If $\v D_t(t)=0$ and $\v D_{t-1}(t)=1$, then $\psi_t(T)\in\{0,1\}$.
Moreover, under the event $\v D_t(t)=0$ and $\v D_{t-1}(t)=1$, if the process governed by $\pi_t$ hits queue length zero in a round after $t$ and no later than $T$, then $\psi_t(T)=0$.
Equivalently, $\psi_t(T)=1$ implies that this process does not hit queue length zero before round $T$.
\end{lemma}

We also use the following standard inequalities.

\begin{lemma}[Conditional Hoeffding lemma]
\label{lem:app_cond_hoeffding}
Let $X$ be a random variable satisfying $X\in[-1,1]$ almost surely.
If $\E[X\mid\cal G]\leq -a$ for some $a\geq0$, then for every $\zeta\geq0$,
\begin{align*}
    \E\sqbr{\exp(\zeta X)\mid\cal G}
    \leq
    \exp\br{-\zeta a+\frac{\zeta^2}{2}}.
\end{align*}
\end{lemma}

\begin{lemma}[Multiplicative Chernoff bound]
\label{lem:app_adapted_chernoff}
Let $Z_1,\ldots,Z_n$ be $\{0,1\}$-valued random variables adapted to a filtration $\{\cal G_i\}_{i=0}^n$.
If $\E[Z_i\mid\cal G_{i-1}]\geq p$ for every $i$, then
\begin{align*}
    \P\br{
        \sum_{i=1}^n Z_i
        <
        \frac{pn}{2}
    }
    \leq
    \exp\br{-\frac{pn}{8}}.
\end{align*}
\end{lemma}

\begin{lemma}[Matrix Chernoff bound]
\label{lem:app_matrix_chernoff}
Let $Y_1,\ldots,Y_n\in\R^d$ be i.i.d. random vectors satisfying $\norm{Y_i}_2\leq1$ almost surely and $\E[Y_iY_i\tp]\succeq\sigma_0^2\I$.
Then
\begin{align*}
    \P\br{
        \lambda_{\min}\br{\sum_{i=1}^n Y_iY_i\tp}
        <
        \frac{n\sigma_0^2}{2}
    }
    \leq
    d\exp\br{-\frac{n\sigma_0^2}{8}}.
\end{align*}
\end{lemma}

\begin{lemma}
\label{lem:app_bernoulli_kl}
Let $u,v\in(0,1)$.
If $w(1-w)\geq m$ for every $w\in[\min\cbr{u,v},\max\cbr{u,v}]$, then
\begin{align*}
    \KL\br{\Bern(u)\|\Bern(v)}
    \leq
    \frac{(u-v)^2}{2m}.
\end{align*}
\end{lemma}

\begin{proof}[Proof of \Cref{lem:app_bernoulli_kl}]
Fix $v\in(0,1)$ and define $f(r)=\KL\br{\Bern(r)\|\Bern(v)}$ for $r\in(0,1)$.
Then $f(v)=0$, $f'(v)=0$, and $f''(r)=1/(r(1-r))$.
By Taylor's theorem, for some point $\bar r$ between $u$ and $v$,
\begin{align*}
    \KL\br{\Bern(u)\|\Bern(v)}
    =
    f(u)
    =
    \frac{(u-v)^2}{2\bar r(1-\bar r)}.
\end{align*}
The point $\bar r$ lies between $u$ and $v$.
Thus $\bar r(1-\bar r)\geq m$ by assumption, and the desired bound follows.
\end{proof}

\begin{lemma}[Bretagnolle--Huber inequality]
\label{lem:app_bh}
For any probability measures $P,Q$ and event $G$,
\begin{align*}
    P(G^c)+Q(G)
    \geq
    \frac12\exp\cbr{-\KL\br{P\|Q}}.
\end{align*}
\end{lemma}

\end{document}